%% file: stochasticPrivacy.tex
\documentclass[letterpaper]{article}
\input{commands}
\title{Stochastic Privacy}
\author{Adish Singla \\ ETH Zurich \\ adish.singla@inf.ethz.ch
\And Eric Horvitz  \\ Microsoft Research \\ horvitz@microsoft.com
\And Ece Kamar     \\ Microsoft Research \\ eckamar@microsoft.com
\And Ryen White    \\ Microsoft Research \\ ryen.white@microsoft.com
}
\begin{document}
\maketitle

\begin{abstract}

\begin{quote}
Online services such as web search and e-commerce applications typically rely on the collection of data about users, including details of their activities on the web.  Such personal data is used to enhance the quality of service via personalization of content and to maximize revenues via better targeting of advertisements and deeper engagement of users on sites.  To date, service providers have largely followed the approach of either requiring or requesting consent for opting-in to share their data.  Users may be willing to share private information in return for better quality of service or for incentives, or in return for assurances about the nature and extend of the logging of data. We introduce \emph{stochastic privacy}, a new approach to privacy centering on a simple concept: A guarantee is provided to users about the upper-bound on the probability that their personal data will be used. Such a probability, which we refer to as \emph{privacy risk}, can be assessed by users as a preference or communicated as a policy by a service provider.  Service providers can work to personalize and to optimize revenues in accordance with preferences about privacy risk.  We present procedures, proofs, and an overall system for maximizing the quality of services, while respecting bounds on allowable or communicated privacy risk. We demonstrate the methodology with a case study and evaluation of the procedures applied to web search personalization.  We show how we can achieve near-optimal utility of accessing information with provable guarantees on the probability of sharing data.  
\end{quote}
\end{abstract}

\vspace{-4mm}
\section{Introduction}

\begin{figure*}[t!]
\centering
\includegraphics[width=0.85\textwidth]{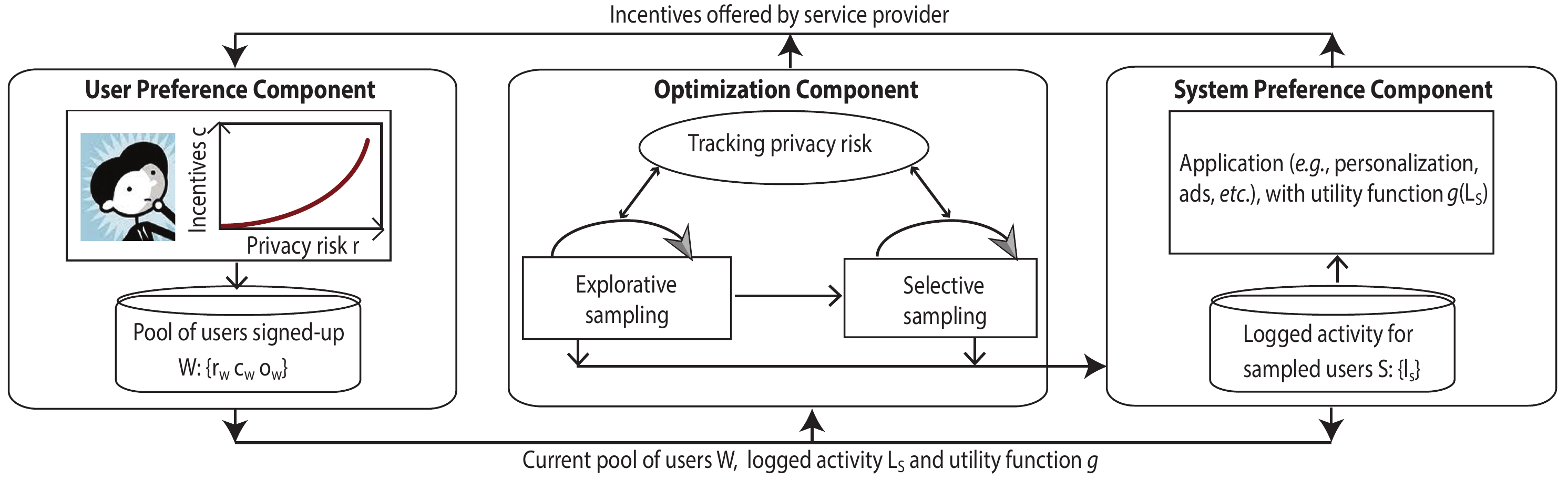}
\caption{Overview of stochastic privacy.}
\label{fig:approach}
\end{figure*}

Online services such as web search, recommendation engines, social networks, and e-commerce applications typically rely on the collection of data about activities ( \emph{e.g.}, click logs, queries, and browsing information) and personal information (\emph{e.g.}, location and demographics) of users. The availability of such data enables providers to personalize services to individuals and also to learn how to enhance the service for all users (\emph{e.g.}, improved search results relevance). User data is also important to providers for optimizing revenues via better targeted advertising, extended user engagement and popularity, and even the selling of user data to third party companies. Permissions are typically obtained via broad consent agreements that request user permission to share their data through system dialogs, or via complex \emph{Terms of Service}.  Such notices are typically difficult to understand and ignored by more than 40 percent of users \cite{2012-technet}. In other cases, a plethora of requests for information such as user location may be shown in system dialogs at run-time or installation time.
Beyond the normal channels for sharing data, potential breaches of information are possible via attacks by malicious third parties and malware, and through surprising situations such as the AOL data release \cite{2006-_aol,2007-_anonymyzing-query-logs} and de-anonymization of released Netflix logs \cite{2008-_netflix-deanonymization}.
The charges by the Federal Trade Commission against Facebook \cite{2011-ftc_facebook} and Google \cite{2012-ftc_google} highlight increasing concerns by privacy advocates and government institutions about the large-scale recording of personal data.

Ideal approaches to privacy in online services would enable users to benefit from machine learning over data from populations of users, yet consider users' preferences as a top priority. Prior research in this realm has focused on designing privacy-preserving methodologies that can provide for control of a privacy-utility tradeoff \cite{2007-_anonymyzing-query-logs,2008-aaai_krause_privacy-personalization}. Research has also explored the feasibility of incorporating user preferences over what type of data can be logged \cite{2007-_privacy-enhancing,2008-_policy-perspective,2005-chi_privacy-preferences,2008-aaai_krause_privacy-personalization}.

We introduce a new approach to privacy that we refer to as \emph{stochastic privacy}.  Stochastic privacy centers on the simple idea of providing a guarantee to users about the maximum likelihood that their data will be accessed and used by a service provider. We refer to this measure as the assessed or communicated \emph{privacy risk}, which may be increased in return for increases in the quality of service or other incentives.  Very small probabilities of sharing data may be tolerated by individuals (just as lightning strikes are tolerated as a rare event), yet offer providers sufficient information to optimize over a large population of users.  Stochastic privacy depends critically on harnessing inference and decision making to make choices about data collection within the constraints of a guaranteed privacy risk.

We explore procedures that can be employed by service providers when preferences about the sharing of data are represented as privacy risk.  The goal is to maximize the utility of service using data extracted from a population of users, while abiding by the agreement reached with users on privacy risk. We show that optimal selection of users under these constraints is NP-hard and thus intractable, given the massive size of the online systems. As a solution, we propose two procedures, \randgreedy and \spgreedy, that combine greedy value of information analysis with obfuscation to offer mechanisms for tractable optimization, while satisfying stochastic privacy guarantees. We present performance bounds for the expected utility achievable by these procedures compared to the optimal solution.
Our contributions can be summarized as follows:

\begin{itemize}
	\item Introduction of stochastic privacy, an approach that represents preferences about the probability that data will be shared, and methods for trading off privacy risk, incentives, and quality of service.
	\item A tractable end-to-end system for implementing a version of stochastic privacy in online services.
	\item \randgreedy and \spgreedy procedures for sampling users under the constraints of stochastic privacy, with theoretical guarantees on the acquired utility.
	\item Evaluation to demonstrate the effectiveness of the proposed procedures on a case study of user selection for personalization in web search.
\end{itemize}

\vspace{-2mm}
\section{Stochastic Privacy Overview}\label{sec:model}

Figure~\ref{fig:approach} provides an overview of stochastic privacy in the context of a particular design of a system that implements the methodology. The design is composed of three main components: (i) a user preference component, (ii) a system preference component, and (iii) an optimization component for guiding the system's data collection. We now provide details about each of the components and then formally specify the optimization problem for \emph{selective sampling} module.
\vspace{-1mm}
\subsection{User Preference Component}

The user preference component interacts with users (\emph{e.g.}, during signup) and establishes an agreement between a user and service provider on a tolerated probability that the user's data will be shared in return for better quality of service or incentives. Representing and capturing users' tolerated privacy risk allows users to move beyond the binary choice of yes or no on the sharing of data.  The incentives offered to users can be personalized based on the metalevel information available for a user  (\emph{e.g.}, general location information inferred from a previously shared IP address) and can vary from guarantees of improved service \cite{2010-jair_krause_privacy-personalization} to complementary software and entries in a lottery to win cash prizes (as done by the comScore service \cite{Wikipedia-comScore}). 

Formally, let $W$ be the population of users signed-up for a service. Each user $w \in W$ is represented with the tuple  $\{r_w,c_w,o_w\}$, where $o_w$ is the metadata information (\emph{e.g.}, IP address) available for user $w$ prior to selecting and logging finer-grained data about the user. $r_w$ is the privacy risk assessed by the user, and $c_w$ is the corresponding incentive provided in return for the user assuming the risk. The elements of this tuple can be updated through interactions between the system and the user. For simplicity of analysis, we shall assume that the pool $W$ and user preferences are static.

\vspace{-1mm}
\subsection{System Preference Component}

The goal of the service provider is to optimize the quality of service.  For example, a provider may wish to personalize web search and to improve the targeting of advertising for maximization of revenue.  The service provider may record the activities of a subset of users (\emph{e.g.}, sets of queries issued, sites browsed, \emph{etc.}) and use this data to provide better service globally or to a specific cohort of users. We model the private data of activity logs of user $w$ by variable $l_w \in 2^L$, where $L$ represents the web-scale space of activities (\emph{e.g.}, set of queries issued, sites browsed, \emph{etc.}) . However, $l_w$ is observed by the system only after $w$ is selected and the data from $w$ is logged. We model the system's uncertain belief of $l_w$ by a random variable $Y_w$, with $l_w$ being its realization distributed according to conditional probability distribution $P(Y_w=l_w|o_w)$. In order to make an informed decision about user selection, the distribution $P(Y_w=l_w|o_w)$ is learned by the system using data available from the user and recorded logs of other users. We quantify the utility of application by logging activities $L_S$ from selected users $S$ through function $g:2^L \rightarrow \mathbb{R}$, given by $g(\bigcup_{s \in S} l_s)$. The expected value of the utility that the system can expect to gain by selecting users $S$ with observed attributes $O_S$ is characterized by distribution $P$ and utility function $g$ as: $\tilde{g}(S) \equiv \mathbb{E}_{Y_S}\big[g(\bigcup_{s \in S} l_s)\big] = \sum_{L_S \in 2^L \times S} \big(P(Y_S = L_S | O_S) \cdot g(\bigcup_{s \in S} l_s)\big)$. However, the application itself may be using the logs $L_S$ in a complex manner (such as training a ranker \cite{2011-sigir_bennett_location-personalization}) and evaluating this on complex user metrics \cite{2013-cikm_ryen_models-of-user-satisfaction}. Hence, the system uses a surrogate utility function $f(S) \approx \tilde{g}(S)$ to capture the utility through a simple metric, for example, coverage of query-clicks obtained from the sampled users \cite{2010-www_singla_click-features} or reduction in uncertainty of click phenomena \cite{2008-aaai_krause_privacy-personalization}. 

In our model, we require that the set function $f$ to be \emph{non-negative},  \emph{monotone} (i.e., whenever $A \subseteq A' \subseteq W$, it holds that $f(A) \leq f(A')$) and \emph{submodular}. Submodularity is an intuitive notion of diminishing returns, stating that, for any sets $A \subseteq A' \subseteq W$, and any given user $a \notin A'$, it holds that $f(A \cup \{a\}) - f(A) \geq f(A' \cup\{a\})-f(A')$. These conditions are general, satisfied by many realistic, as well as complex utility functions \cite{2007-aaai_krause_observation-selection}, such as reduction in click entropy \cite{2008-aaai_krause_privacy-personalization}. As a concrete example, consider the setting where attributes $O$ represent geo-coordinates of the users and $D:O \times O \rightarrow \mathbb{R}$ computes the geographical distance between any two users. The goal of the service provider is to provide location-based personalization of web search. For such an application, click information from local users provides valuable signals for personalizing search  \cite{2011-sigir_bennett_location-personalization}. The system's goal is to select a set of users $S$, and to leverage data from these users to enhance the service for the population. For search queries originating from any other user $w$, it uses the click data from the nearest user in $S$, given by $\operatorname*{arg\,min}_{s \in S} D(o_s, o_w)$. One approach for finding such a set $S$ is solving the \emph{k-medoid} problem which aims to minimize the sum of pairwise distances between selected set and the remaining population \cite{2013-nips_krause_distributed,2009-book_kaufman_finding}. Concretely, this can be captured by the following submodular utility function:   
\vspace{-2mm}
\begin{align}\label{eq:utilfunction}
f(S) = \frac{1}{|W|} \sum_{w \in W} \Big(\operatorname*{min}_{x \in X} D(o_x, o_w) - \operatorname*{min}_{s \in S \cup X} D(o_s, o_w)\Big)
\vspace{-5mm}
\end{align}
Here, $X$ is any one (or a set of) fixed reference location(s), for example, simply representing origin coordinates and is used ensure that function $f$ is non-negative and  monotone.
Lemma~\ref{lemma:utilfunction} formally states the properties of this function.
\begin{table*}[t!]
\centering
\begin{tabular}{|l|c|c|l|}
\hline
{\bf Procedure}   & {\bf Competitive utility} & {\bf Privacy guarantees}	& {\bf Polynomial runtime} \\ \hline \hline
{\opt}         	  & {\cmark}             	 & {\xmark}         & {{\xmark} $\mathcal{O}\big(|W|^B\big)$} \\ \hline
{\greedy}         & {\cmark}             	 & {\xmark}         & {{\cmark} $\mathcal{O}\big(B \cdot |W|\big)$} \\ \hline
{\random}         & {\xmark}             	 & {\cmark}         & {{\cmark} $\mathcal{O}\big(B\big)$}	 \\ \hline
{\bf \randgreedy} & {\cmark}             	 & {\cmark}         & {{\cmark} $\mathcal{O}\big(B \cdot |W| \cdot r\big)$} \\ \hline
{\bf \spgreedy}   & {\cmark}             	 & {\cmark}         & {{\cmark} $\mathcal{O}\big(B \cdot |W| \cdot log(\sfrac{1}{r})\big)$} \\ \hline
\end{tabular}
\caption{Properties of different procedures. \randgreedy and \spgreedy satisfy all the desirable properties.}\label{tab:algorithms}
\end{table*}

\vspace{-1mm}
\subsection{Optimization Component}

To make informed decisions about data access, the system computes the expected value of information (VOI) of logging the activities of a particular user, \emph{i.e.}, the marginal utility that the application can expect by logging the activity of this user \cite{2008-aaai_krause_privacy-personalization}.
In the absence of sufficient information about user attributes, the VOI may be small, and hence needs to be learned from the data. The system can randomly sample a small set of users from the population that can be used to learn and improve the models of VOI computation (\emph{explorative sampling} in Figure~\ref{fig:approach}). For example, for optimizing the service for a user cohort speaking a specific language, the system may choose to collect logs from a subset of users to learn how languages spoken by users map to geography.
If preferences about privacy risk were not being regarded, VOI can be used to select which users to log with a goal of maximizing the utility for the service provider (\emph{selective sampling} in Figure~\ref{fig:approach}). Given that the utility function of the system is submodular, a greedy selection rule makes near-optimal decisions about data access \cite{2007-aaai_krause_observation-selection}. However, this simple approach could violate the privacy guarantees made with users. To act in accordance with the assessed privacy risk, we design selective sampling procedures that couple obfuscation with VOI analysis to select the set of users to provide data.   

The system needs to ensure that both the explorative and selective sampling approaches respect the privacy guarantees made to users: the likelihood of sampling any user $w$ throughout the execution of the system must be less than the privacy risk factor $r_w$. The system tracks the sampling risk (likelihood of sampling) that user $w$ faces during phases of the execution of explorative sampling, denoted $r^{ES}_w$, and selective sampling, denoted $r^{SS}_w$. The privacy guarantee for a user is preserved as long as: $r_w - \big(1 - (1-r^{ES}_w) \cdot (1-r^{SS}_w)\big) \geq 0$. This difference between the assessed risk and risk faced by a user can be viewed as the \emph{sampling budget} of that user. 


\vspace{-1mm}
\subsection{Optimization Problem for Selective Sampling}
We now focus primarily on the selective sampling module and formally introduce the optimization problem. The goal is to design a sampling procedure $M$ that abides by guarantees of stochastic privacy, yet optimizes the utility of the application in decisions about accessing user data. Given a budget constraint $B$, the goal is to select users $S^M$:
\vspace{-1mm}
\begin{align}\label{eq:opt}
S^M &= \operatorname*{arg\,max}_{S \subseteq W} f(S)\\ 
\vspace{-2mm}
\text{ subject to } & \sum_{s \in S} c_s \leq B \text{ and } r_w - r_w^M \geq 0\text{ }\forall w \in W  \notag.
\vspace{-5mm}
\end{align}
Here, $r_w^M$ is the likelihood of selecting $w \in \mathcal{W}$ by procedure $M$ and hence $r_w - r_w^M \geq 0$ captures the constraint of stochastic privacy guarantee for $w$. Note that we interchangeably write utility acquired by procedure as $f(M)$ to denote $f(S^M)$ where $S^M$ is the set of users selected by running $M$. We shall now consider a simpler setting of constant privacy risk rate $r$ for all users and unit cost per user (thus reducing the budget constraint to a simpler cardinal constraint, given by $|S| \leq B$). These assumptions lead to defining $B \leq W \cdot r$, as that is the maximum possible set size that can be sampled by any procedure for Problem~\ref{eq:opt}.

\vspace{-2mm}
\section{Selective Sampling with Stochastic Privacy}

We shall now propose desiderata of the selection procedures, discuss the hardness of the problem and review several different tractable approaches, as summarized in Table~\ref{tab:algorithms}.
\vspace{-1mm}
\subsection{Desirable Properties of Sampling Procedures}

The problem defined by Equation \ref{eq:opt} requires solving an NP-hard discrete optimization problem, even when stochastic privacy constraint is removed. The algorithm for finding the optimal solution of this problem without the privacy constraint, referred as $\opt$, is intractable \cite{1998-_feige_threshold-of-ln-n}. We address this intractability by exploiting the submodular structure of the utility function $f$ and offer procedures providing provable near-optimal solutions in polynomial time. We aim at designing procedures that satisfy the following desirable properties: \emph{(i)} provides competitive utility w.r.t. \opt with provable guarantees, \emph{(ii)} preserves stochastic privacy guarantees, and \emph{(iii)} runs in polynomial time.

\vspace{-1mm}
\subsection{Random Sampling: \random}

\random simply samples the users at random, without any consideration of cost and utility. The likelihood of any user $w$ to be selected by the algorithm is $r_w^{\random} = \sfrac{B}{W}$ and hence privacy risk guarantees are trivially satisfied since $B \leq W \cdot r$ as defined in Problem~\ref{eq:opt}). In general, \random can perform arbitrarily poorly in terms of acquired utility, specifically for applications targeting particular user cohorts.

\vspace{-1mm}
\subsection{Greedy Selection: \greedy}

Next, we explore a greedy sampling strategy that maximizes the expected marginal utility at each iteration to guide the decision about selecting a next user to log. Formally, \greedy starts with empty set $S=\emptyset$. At an iteration $i$, it greedily selects a user $s^*_i = \operatorname*{arg\,max}_{w \subseteq W \setminus S} f(S \cup {w}) - f(S)$ and adds it to the current selection of users $S = S \cup \{s^*_i\}$. It stops when $|S| = B$. 

A fundamental result by \citet{1978-_nemhauser_submodular-max} states that the utility obtained by this greedy selection strategy is guaranteed to be at least $(1 - \sfrac{1}{e})$ $(=0.63)$ times that obtained by \opt. This result is tight under reasonable complexity assumptions ($P \neq NP$) \cite{1998-_feige_threshold-of-ln-n}. However, such a greedy selection clearly violates the stochastic privacy constraint in Problem~\ref{eq:opt}---consider the user $w^*$ with highest marginal value: $w^* = \operatorname*{arg\,max}_{w \subseteq W} f({w})$. The likelihood that this user will be selected by the algorithm $r_{w^*}^{\greedy} = 1$, regardless of the requested privacy risk $r_{w^*}$.

\vspace{-1mm}
\subsection{Sampling and Greedy Selection: \randgreedy}
We combine the ideas of \random and \greedy to design procedure \randgreedy which provides guarantees on stochastic privacy and competitive utility.  \randgreedy is an iterative procedure that samples a small batch of users $\psi(s)$ at each iteration, then greedily selects $s^* \in \psi(s)$ and removes the entire set $\psi(s)$ for further consideration. By keeping the batch size $\psi(s) \leq \sfrac{W \cdot r}{B}$, the procedure ensures that the privacy guarantees are satisfied. As our user pool $W$ is static, to reduce complexity, we consider a simpler version of \randgreedy that defers the greedy selection. Formally, this is equivalent to first sampling the users from $W$ at rate $r$ to create a subset $\widetilde{W}$ such that $|\widetilde{W}| = |W| \cdot r$, and then, running the \greedy algorithm on $\widetilde{W}$ to greedily select a set of users of size $B$.

The initial random sampling ensures a guarantee on the privacy risk for users during the execution of the procedure. In fact, for any user $w \in W$, the likelihood of $w$ being sampled and included in subset $\widetilde{W}$ is $r_{w}^{\randgreedy} \leq r$. We further analyze the utility obtained by this procedure in the next section and show that, under reasonable assumptions, the approach can provide competitive utility compared to \opt.

\vspace{-1mm}
\subsection{Greedy Selection with Obfuscation: \spgreedy}
\vspace{-1mm}
\spgreedy uses an inverse approach of mixing \random and \greedy: it does greedy selection, followed by obfuscation,  as illustrated in Procedure~\ref{alg:spgreedy}. It assumes an underlying distance metric $D:W \times W \rightarrow \mathbb{R}$ which captures the notion of distance or dissimilarity among users. As in \greedy, it operates in iterations and selects the element $s^*$ with maximum marginal utility at each iteration. However, to ensure stochastic privacy, it obfuscates $s^*$ with similar users using distance metric $D$ to create a set $\psi(s^*)$ of size $\sfrac{1}{r}$, then samples one user randomly from $\psi(s^*)$ and removes the entire set $\psi(s^*)$ for further consideration.

The guarantees on privacy risk hold by the following arguments: During the execution of the algorithm, any user $w$ becomes a possible candidate of being selected if the user is part of $\psi(s^*)$ in some iteration (\emph{e.g.}, iteration $i$).  Given that $|\psi(s^*)| \geq \sfrac{1}{r}$ and algorithm randomly sample $v \in \psi(s^*)$, the likelihood of $w$ being selected in iteration $i$ is at most $r$. The fact that set $\psi(s^*)$ is removed from available pool $\widetilde{W}$ at the end of the iteration ensures that $w$ can become a possible candidate of selection only once.
\vspace{-1mm}
\begin{algorithm}[h!]
\nl {\bf Input}: \emph{users}~$W$; \emph{cardinality constraint}~$B$; \emph{privacy risk}~$r$; \emph{distance metric} $D:W \times W \rightarrow \mathbb{R}$\;
\nl {\bf Initialize}:\\
\begin{itemize}
\item {\bf Outputs}: \emph{selected users}~$S \leftarrow \emptyset$;
\item {\bf Variables}: \emph{remaining users}~$W' \leftarrow W$;
\end{itemize}
\Begin{
	\nl  \While{$|S|\leq B$}{
	\nl		  $s^* \leftarrow \operatorname*{arg\,max}_{w \in W'} f(S \cup {w}) - f(S)$\; 
	\nl       Set $\psi(s^*) \leftarrow s^*$\;
	\nl       \While{$|\psi(s^*)| < \sfrac{1}{r}$}{
	\nl			$v \leftarrow \operatorname*{arg\,min}_{w \in W' \setminus \psi(s^*)} D(w, s^*)$\;
	\nl			$\psi(s^*) \leftarrow \psi(s^*)  \cup \{v\}$\;
          	  }
	\nl	  Randomly select $\widetilde{s^*} \in \psi(s^*)$\;
	\nl	  $S \leftarrow S \cup \{\widetilde{s^*}\}$\;
	\nl	  $W' \leftarrow W'  \setminus \psi(s^*)$\;
    		}
}
\nl {\bf Output}: $S$\\
\caption{\spgreedy}
\label{alg:spgreedy} 
\end{algorithm}

\vspace{-4mm}
\section{Performance Analysis}
\vspace{-2mm}
We now analyze the performance of the proposed procedures in terms of the utility acquired compared to that of the \opt as baseline. We first analyze the problem in a general setting and then under a set of practical assumptions on the structure of underlying utility function $f$ and population of users $W$. The proofs of all the results are available at \footnote{Available anonymously at:\\ \url{http://tinyurl.com/aaai-stocpriv-longer}}.
\vspace{-2mm}
\subsection{General Case}

In the general setting, we show that one cannot do better than $r \cdot f(OPT)$ in the worst case. Consider a population of users $W$ where only one user $w^*$ has utility value of 1, and rest of the users $W \setminus {w^*}$ have utility of 0. The \opt gets a utility of $1$ by selecting $S^\opt = \{w^*\}$. Consider any procedure $M$ that has to respect the guarantees on privacy risk. If the privacy rate of $w^*$ is $r$, then $M$ can select $w^*$ with only a maximum probability of $r$. Hence, the maximum expected utility that any procedure $M$ for Problem~\ref{eq:opt} can achieve is $r$.

On a positive note, a trivial algorithm can always achieve a utility of $(1 - \sfrac{1}{e}) \cdot r \cdot f(OPT)$ in expectation. This result can be achieved by running \greedy to select a set $S^\greedy$ and then choosing the final solution to be $S^\greedy$ with probability $r$, and else output an empty set. Theorem~\ref{theorem:general} formally states these results for the general problem setting.
\vspace{-1mm}
\begin{theorem}\label{theorem:general}
Consider the Problem~\ref{eq:opt} of optimizing a submodular function $f$ under cardinality constraint $B$ and privacy risk rate $r$. For any distribution of marginal utilities of population $W$, a trivial procedure can achieve expected utility of at least $(1 - \sfrac{1}{e}) \cdot r \cdot f(OPT)$. In contrast, there exists an underlying distribution for which no procedure can have expected utility of more than $r \cdot f(OPT)$. 
\end{theorem}

\vspace{-1mm}
\subsection{Smoothness and Diversification Assumptions}
\vspace{-1mm}
In practice, we can hope to do much better than the worst-case results described in Theorem~\ref{theorem:general} by exploiting the underlying structures of users attributes and utility function. We start with the assumption that there exists a distance metric $D:W \times W \rightarrow \mathbb{R}$ which captures the notion of distance or dissimilarity among users. For any given $w \in W$, let us define its ${\alpha}$-neighborhood to be the set of users within a distance $\alpha$ from $w$ (i.e., ${\alpha}$-close to $w$): $N_{\alpha}(w) = \{v : D(v, w) \leq \alpha\}$. We assume that population of users is large and that the number of users in the $N_{\alpha}(w)$ is large. We formally capture these requirements in Theorems~\ref{theorem:randgreedy},\ref{theorem:spgreedy}.

Firstly, we consider utility functions that change gracefully with changes in inputs, similar to the notion of $\lambda$\emph{-Lipschitz} set functions used in \citet{2013-nips_krause_distributed}. We formalize the notion of smoothness in the utility function $f$ w.r.t metric $D$ as follows:

\begin{definition}\label{definition:smooth}
For any given set of users $S$, let us consider a set $\widetilde{S}_\alpha$ obtained by replacing every $s \in S$ with any $w \in N_{\alpha}(s)$. Then, $|f(S) - f(\widetilde{S}_\alpha)| \leq \lambda_{f} \cdot \alpha \cdot |S|$, where parameter $\lambda_{f}$ captures the notion of smoothness of function $f$.  
\end{definition}

Secondly, we consider utility functions that favor diversity or dissimilarity of users in the subset selection w.r.t the distance metric $D$. We formalize this notion of diversification in the utility function as follows:

\begin{definition}\label{definition:diverse}
Let us consider any given set of users $S \subseteq W$ and a user $w \in W$. Let $\alpha = \operatorname*{min}_{s \in S} D(s, w)$. Then, $f(S \cup {w}) - f(S) \leq \Upsilon_{f} \cdot \alpha$, where parameter $\Upsilon_{f}$ captures the notion of diversification of function $f$.  
\end{definition}

The utility function $f$ introduced in Equation~\ref{eq:utilfunction} satisfies both the above assumptions as formally stated below.

\begin{lemma}\label{lemma:utilfunction}
Consider the utility function $f$ in Equation~\ref{eq:utilfunction}. $f$ is submodular, and satisfies the properties of smoothness and diversification, i.e. has bounded $\lambda_{f}$ and $\Upsilon_{f}$.
\end{lemma}
We note that for the functions with unbounded $\lambda$ and $\Upsilon$ (\emph{i.e.}, $\lambda_{f} \rightarrow \infty$ and $\Upsilon_{f} \rightarrow \infty$), it would lead to the general problem settings (equivalent to no assumptions) and hence results of Theorem~\ref{theorem:general} apply.

\begin{figure*}[t!]
\centering
   \subfigure[Vary budget $B$]{
     \includegraphics[width=0.32\textwidth]{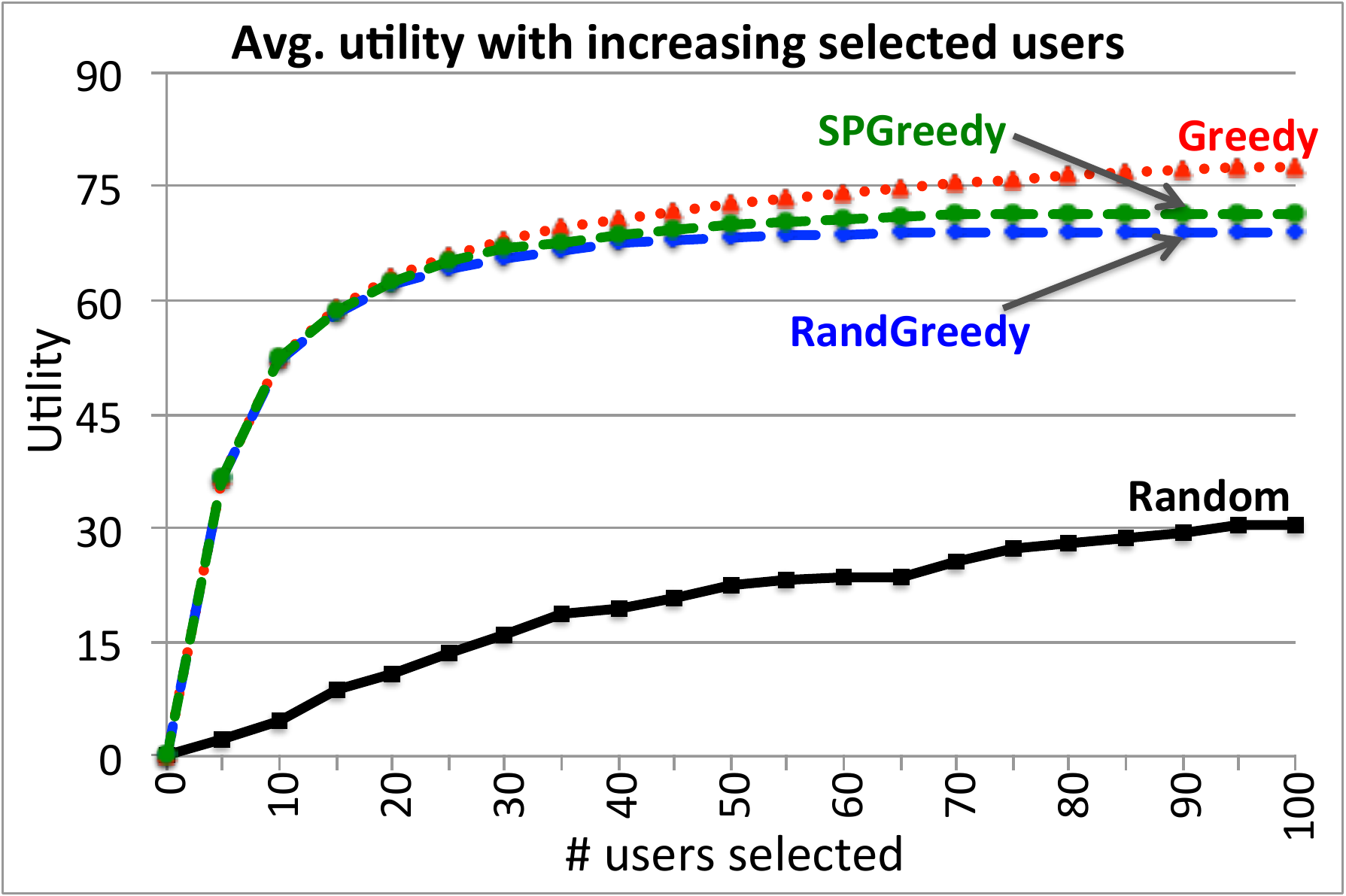}
     \label{fig:results_vary-budget}
   }
   \subfigure[Vary privacy risk $r$]{
     \includegraphics[width=0.32\textwidth]{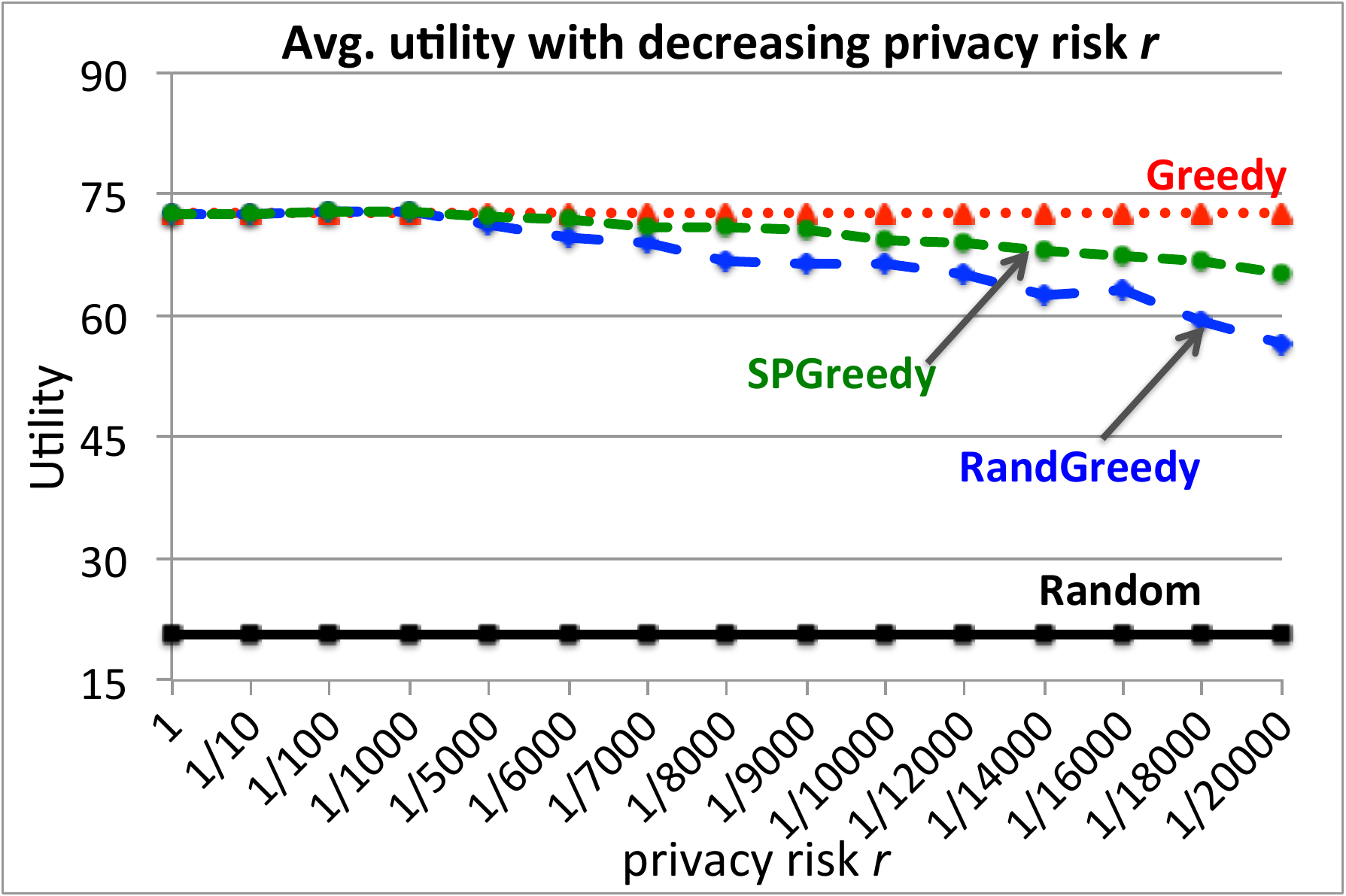}
     \label{fig:results_vary-privacy}
   }
   \subfigure[Analyze \spgreedy]{
     \includegraphics[width=0.32\textwidth]{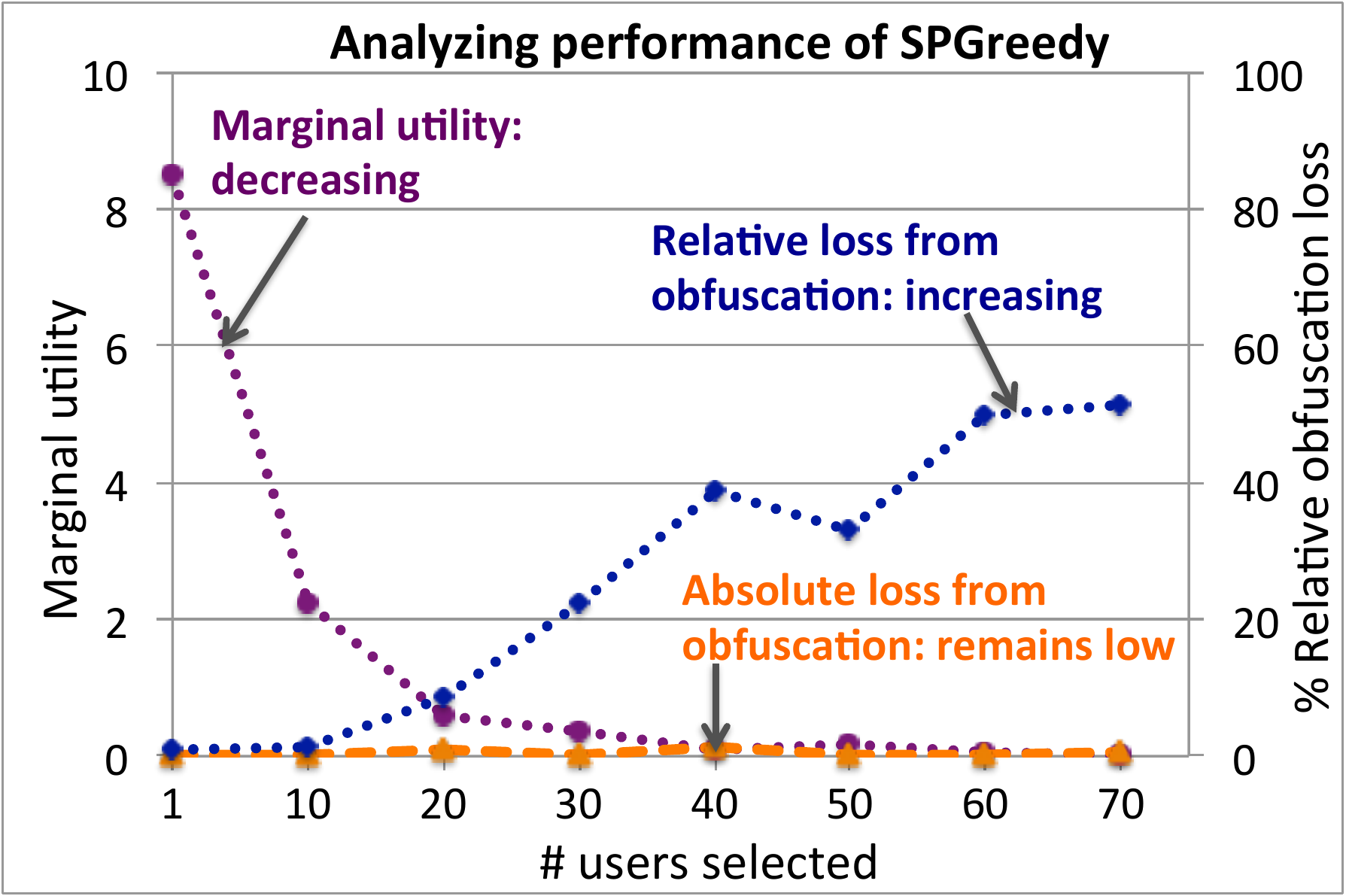}
     \label{fig:results_analyze-spgreedy}
   }
\caption{In Fig.~\ref{fig:results_vary-budget}, for a fixed $r=\sfrac{1}{10000}$, budget $B$ or number of users selected in increased, showing the competitiveness of our procedures w.r.t \greedy. In Fig.~\ref{fig:results_vary-privacy}, a fixed $B=50$ is used, and level of privacy risk $r$ is reduced. The results demonstrate that the performance of \randgreedy and \spgreedy degrades smoothly. Fig.~\ref{fig:results_analyze-spgreedy} analyze the execution of the procedure \spgreedy and illustrates that the loss incurred in marginal utility at every step via obfuscation is very low.
}
\label{fig:results}
\end{figure*}

\vspace{-1mm}
\subsection{Performance Bounds}
Under the assumption of smoothness (\emph{i.e.}, bounded $\lambda_{f})$, we can show the following bound on utility of {\randgreedy}:

\begin{theorem}\label{theorem:randgreedy}
Consider the Problem~\ref{eq:opt} for function $f$ with bounded $\lambda_{f}$. Let $S^\opt$ be the set returned by $\opt$ for Problem~\ref{eq:opt} after relaxing privacy constraints. For a desired $\epsilon < 1$, let $\alpha_{rg} = \operatorname*{arg\,min}_{\alpha}\{\alpha : |N_{\alpha}(s)| \geq \sfrac{1}{r} \cdot \log(\sfrac{B}{\epsilon})\text{ }\forall s \in S^\opt, \text{ where } N_{\alpha}(s_i) \cap N_{\alpha}(s_j) = \emptyset\text{ }\forall s_i, s_j \in S^\opt \}$. Then, with probability at least $(1 - \epsilon)$, 
$$\mathbb{E}[f(\randgreedy)] \geq (1 - \sfrac{1}{e}) \cdot \big(f(\opt) - \alpha_{rg} \cdot \lambda_{f} \cdot B\big).$$
\end{theorem}
\vspace{-1mm}
Under the assumption of smoothness and diversification (\emph{i.e.}, bounded $\lambda_{f}$ and $\Upsilon_{f}$), we can show the following bound on utility of {\spgreedy}:
\begin{theorem}\label{theorem:spgreedy}
Consider the Problem~\ref{eq:opt} for function $f$ with bounded $\lambda_{f}$ and $\Upsilon_{f}$. Let $S^\greedy$ be the set returned by $\greedy$ for Problem~\ref{eq:opt}. Let $\alpha_{spg} = \operatorname*{arg\,min}_{\alpha}\{\alpha :|N_{\alpha}(s)| \geq \sfrac{1}{r} \text{ } \forall s \in S^\greedy\}$. Then, 
$$\mathbb{E}[f(\spgreedy)] \geq (1 - \sfrac{1}{e}) \cdot f(\opt) - (2 \cdot \lambda_{f} + \Upsilon_{f})\cdot \alpha_{spg} \cdot B.$$
\end{theorem}
\vspace{-1mm}
Intuitively, these results imply that both \randgreedy and \spgreedy achieve competitive utility w.r.t \opt, and the performance degrades smoothly as the privacy risk $r$ is decreased or the bounds on smoothness and diversification increase.

\vspace{-2mm}
\section{Experimental Evaluation}
We shall now report on experiments we performed to build insights about  the performance of the stochastic privacy procedures on a case study of the selective collection of user data in support of web search personalization.

\vspace{-1mm}
\subsection{Benchmarks and Metrics}
We compare the performance of the \randgreedy and \spgreedy procedures against the baselines of \random and \greedy. While \random provides a trivial lower benchmark for any procedure, \greedy is a natural upper bound on the utility, given that the \opt itself is intractable.
To analyze the robustness of our procedures, we then vary the level of privacy risk $r$. We further carried out experiments to understand the loss incurred from obfuscation phase during the execution of \spgreedy.

\vspace{-1mm}
\subsection{Experimental Setup}
 We considered the application of providing location-based personalization for queries issued for the business domain (\emph{e.g.}, real-estate, financial services, \emph{etc.}). The goal is to select a set of users $S$ who are expert web search users in this domain.  We seek to leverage the click data from these users to improve the relevance of search results shown to those searching for local businesses. The experiments are based on using a surrogate utility function as introduced in Equation~\ref{eq:utilfunction}. As we are interested in specific domain of business-related queries, we modify the utility function in Equation~\ref{eq:utilfunction} by restricting $S$ to users who are experts in the domain, as further described below.  The acquired utility can be interpreted as the average reduction in the distance for any user $w$ in the population to the nearest expert $s \in S$.

The primary source of data for this study is obtained from interaction logs on a major web search engine. We considered a fraction of users who issued at least one query in month of October 2013, restricted to queries coming from IP addresses located within ten neighboring states in the western region of the United States. This resulted in a pool $W$ of seven million users. We considered a setting where system has access to metadata information of geo-coordinates of the users, as well as a probe of the last 20 search-result clicks for each user, which together constitutes the observed attributes of user denoted as $o_w$. Each of these clicks are then classified into a topical hierarchy from a popular web directory, the Open Directory Project (ODP) (dmoz.org), using automated techniques \cite{2010-www_bennett_classification}. With a similar objective to \citet{2009-wsdm_white_experts}, the system then uses this classification to identify users who are expert in the business domain.
We used a simple rule of classifying a user as an expert if at least one click was issued in the domain of interest. With this, the system marks a set of users $W' \subseteq W$ as experts, and the set $S$ in Equation~\ref{eq:utilfunction} is restricted to $W'$. We note that the specific thresholds or variable choices do not affect the overall results below.

\vspace{-2mm}
\subsection{Results}
\vspace{-1mm}
We now discuss the findings from our experiments.

{\bf Varying the budget $B$:}
In our first set of experiments, we vary the budget $B$, or equivalently the number of users selected, and measured the utility acquired by different procedures. The privacy risk rate is set fixed to $r = \sfrac{1}{10000}$. Figure~\ref{fig:results_vary-budget} illustrates that both \randgreedy and \spgreedy are competitive w.r.t \greedy and clearly outperform naive baseline of \random.

{\bf Varying the privacy risk $r$:}
We then vary the level of privacy risk, for a fixed budget $B=50$, to measure the robustness of the \randgreedy and \spgreedy.  The results in Figure~\ref{fig:results_vary-privacy} demonstrate that the performance of \randgreedy and \spgreedy degrades smoothly, as per the performance analysis in Theorems~\ref{theorem:randgreedy},\ref{theorem:spgreedy}.

{\bf Analyzing performance of \spgreedy:}
Lastly, we perform experiments to understand the execution of \spgreedy and the loss incurred from the obfuscation step. \spgreedy removes $\sfrac{1}{r}$ users from pool at every iteration. As a result, for small privacy risk $r$, the relative loss from obfuscation (\emph{i.e.}, relative \% difference in marginal utility acquired by a user chosen by greedy selection, compared to one finally picked after obfuscation) could possibly increase over the execution of procedure as illustrated in Figure~\ref{fig:results_vary-budget}, using a moving average of window size 10. However, the diminishing returns property of the utility ensures that \spgreedy incurs very low absolute loss in marginal utility from obfuscation at every step.

\vspace{-2mm}
\section{Discussion}

We introduced \emph{stochastic privacy}, a new approach to privacy that centers on service providers abiding by guarantees about not exceeding a specified likelihood of logging data, and maximizing information collection in accordance with these guarantees.  We presented procedures and an overall system design for maximizing the quality of services while respecting privacy risks agreed with populations of users. 


Directions for this research include the assessments of user preferences about the probability of sharing data, including how users trade increases in privacy risk for enhanced service and monetary incentives. Directions also include exploring the rich space of designs for interactive and longer-term controls and settings of a tolerated risk of sharing data. Opportunities include policies and analyses based on the sharing of data as a privacy risk \emph{rate} over time.  As an example, systems might one day consider decisions about logging one or more search sessions of users where privacy risk is assessed in terms of the risk of sharing search sessions over time. In another research direction, designs can include models where users are notified when they are selected to share data and are provided with a special reward and option of declining to share at that time.  Iterative analyses can be developed where subsets of users are actively engaged with the option to assume higher levels of privacy risk or to simply provide additional information in return for special incentives. Inferences about the likely preferences on privacy risk and about incentives for subpopulations could be folded into the selection procedures.
\clearpage
{
\fontsize{9.5pt}{10.5pt}
\selectfont
\bibliographystyle{aaai}
\bibliography{stochasticPrivacy}  
}
\onecolumn
\input{appendix_stochasticPrivacy}

\end{document}

%% file: commands.tex
\usepackage{aaai}
\usepackage{times}
\usepackage{helvet}
\usepackage{courier}
\usepackage{savesym}
\usepackage{multirow}
\usepackage{url}
\usepackage[titletoc]{appendix}
\usepackage{xspace}
\usepackage[tight]{subfigure}
\usepackage{graphicx}
\usepackage{url}
\usepackage{hyperref}
\usepackage{enumitem}
\usepackage{color}
\makeatletter
\newif\if@restonecol
\makeatother

\usepackage[lined, boxed, ruled, commentsnumbered, noend]{algorithm2e}
\usepackage{amsmath,amssymb,xfrac}
\usepackage{amsthm}
\setitemize{noitemsep,topsep=1pt,parsep=1pt,partopsep=1pt, leftmargin=12pt}
\setlist{nolistsep}
\newtheorem{lemma}{Lemma}
\newtheorem{theorem}{Theorem}

\newtheorem{definition}{Definition}

\makeatletter

\newcommand{\Rmnum}[1]{\expandafter\@slowromancap\romannumeral #1@}
\makeatother
\usepackage{caption}

\newcommand{\citet}[1]{\citeauthor{#1} (\citeyear{#1})}

\newcommand{\random}{\textsc{Random}\xspace}
\newcommand{\greedy}{\textsc{Greedy}\xspace}
\newcommand{\opt}{\textsc{Opt}\xspace}
\newcommand{\randgreedy}{\textsc{RandGreedy}\xspace}
\newcommand{\spgreedy}{\textsc{SPGreedy}\xspace}
\usepackage{pifont}
\newcommand{\cmark}{\ding{51}}%
\newcommand{\xmark}{\ding{55}}%

%% file: appendix_stochasticPrivacy.tex
\appendix 
{\allowdisplaybreaks
\section{Proof of Lemma~\ref{lemma:utilfunction}}
We prove Lemma~\ref{lemma:utilfunction} by proving three other Lemmas~\ref{lemma:utilfunction:submodular}~\ref{lemma:utilfunction:smoothness}~\ref{lemma:utilfunction:diversification} that are not in the main paper. In Lemma~\ref{lemma:utilfunction:submodular}, by using the decomposable property of the function $f$ from Equation~\ref{eq:utilfunction}, we prove that the function $f$ is non-negative, monotonous (non-decreasing) and submodular. Then, we show that the function satisfies the properties of smoothness (in Lemma~\ref{lemma:utilfunction:smoothness}) and diversification (in Lemma~\ref{lemma:utilfunction:diversification}) by showing an upper bound on the values of the parameters  $\lambda_{f}$ and $\Upsilon_{f}$.

\begin{lemma} \label{lemma:utilfunction:submodular}
Utility function $f$ in Equation~\ref{eq:utilfunction} is non-negative, monotone (non-decreasing)  and submodular.
\end{lemma}
\begin{proof}[\bf{Proof}]
We begin by noting that $f$ is \emph{decomposable}, \emph{i.e.}, it can be written as a sum of simpler functions $f_w$ as: 
\begin{align}
f(S)=\sum_{w \in W} f_w(S) \label{appendix:eq:utilfunctiondecompose}
\end{align}
\noindent
where $f_w(S)$ is given by:
\begin{align}\label{appendix:eq:utilfunctionsimple}
f_w(S) = \frac{1}{|W|} \Big(\operatorname*{min}_{x \in X} D(o_x, o_w) - \operatorname*{min}_{s \in S \cup X} D(o_s, o_w)\Big)
\end{align}
Next, we prove that each of these functions $f_w$ is non-negative, non-decreasing and submodular. To prove that the function is non-decreasing, consider any two sets $S \subseteq S' \subseteq W$. Then,
\begin{align}
f_w(S') - f_w(S) &= \frac{1}{|W|} \Big(\operatorname*{min}_{s \in S \cup X} D(o_s, o_w) - \operatorname*{min}_{s \in S' \cup X} D(o_s, o_w)\Big) \notag \\
&\geq 0 \label{step:utilfunction:submodular:1}
\end{align}
In step~\ref{step:utilfunction:submodular:1}, the inequality holds as the distance to the nearest user for $w$ in $S'$ cannot be more than that in $S$, hence proving that $f_w$ is non-decreasing. Also, it is easy to see that $f_w(\emptyset)$ = 0, which along with the non-decreasing property, ensures that the function $f_w$ is non-negative.
\\

\noindent
To prove that the function is submodular, consider any two sets $S \subseteq S' \subseteq W$, and any given user $v \in W \setminus S'$. When $f_w(S' \cup \{v\}) - f_w(S') = 0$, submodularity holds trivially as we have $f_w(S \cup \{v\}) - f_w(S) \geq 0$ using non-decreasing property. Let us consider the case when $f_w(S' \cup \{v\}) - f_w(S') > 0$, \emph{i.e.}, $v$ is assigned as the nearest user to $w$ from the set $S' \cup \{v\}$, given by $v =
\operatorname*{min}_{s \in S' \cup \{v\} \cup X} D(o_s, o_w)$. In this case, it would also be the case that $v$ is the nearest user to $w$ from the set $S \cup \{v\}$. Then, we can write down the difference of marginal gains as follows:
\begin{align}
\Big(f_w(S' \cup \{v\}) - f_w(S')\Big) &- \Big(f_w(S \cup \{v\}) - f_w(S)\Big) \notag \\ 
&= \bigg(\frac{1}{|W|} \Big(D(o_v, o_w) - \operatorname*{min}_{s \in S' \cup X} D(o_s, o_w)\Big)\bigg) - \bigg(\frac{1}{|W|} \Big(D(o_v, o_w) - \operatorname*{min}_{s \in S \cup X} D(o_s, o_w)\Big)\bigg) \notag \\
&= \frac{1}{|W|} \Big(\operatorname*{min}_{s \in S \cup X} D(o_s, o_w) - \operatorname*{min}_{s \in S' \cup X} D(o_s, o_w)\Big) \notag \\
&\leq 0 \label{step:utilfunction:submodular:2}
\end{align}
In step~\ref{step:utilfunction:submodular:2}, the inequality holds as the function is non-decreasing, thus showing that the marginal gains diminish and hence proving the submodularity of the function $f_w$.
\\

\noindent
By using the fact that these properties are preserved under linear combination with non-negative weights (all equal to $1$ from Equation~\ref{appendix:eq:utilfunctiondecompose}), $f$ is non-negative, non-decreasing and submodular.
\end{proof}

\begin{lemma} \label{lemma:utilfunction:smoothness}
Utility function $f$ in Equation~\ref{eq:utilfunction} satisfies the properties of smoothness, i.e. has bounded $\lambda_{f}$.
\end{lemma} 
\begin{proof}[\bf{Proof}]
For any given set of users $S$, let us consider a set $\widetilde{S}_\alpha$ obtained by replacing every $s \in S$ with any $w \in N_{\alpha}(s)$. The goal is to show that $|f(S) - f(\widetilde{S}_\alpha)| \leq \lambda_{f} \cdot \alpha \cdot |S|$ always holds for a fixed and bounded $\lambda_{f}$.
\\

\noindent
Let us again use the simpler functions $f_w$ from decomposition of $f$ in Equation~\ref{appendix:eq:utilfunctiondecompose} and consider the difference $|f_w(S) - f_w(\widetilde{S}_\alpha)|$. Then,
\begin{align}
|f_w(S) - f_w(\widetilde{S}_\alpha)| \leq \frac{\alpha}{|W|} \label{step:utilfunction:smoothness:1}
\end{align}
In step~\ref{step:utilfunction:smoothness:1}, the inequality holds as the deviation in distance to the nearest user for $w$ in $\widetilde{S}_\alpha$ cannot be more than $\alpha$.  Using this result, we have
\begin{align}
|f(S) - f(\widetilde{S}_\alpha)| &= |\sum_{w \in W} f_w(S) - \sum_{w \in W} f(\widetilde{S}_\alpha)| \notag \\
                                 &\leq \sum_{w \in W} |f_w(S) - f_w(\widetilde{S}_\alpha)| \notag \\
                                 &\leq \sum_{w \in W} \frac{\alpha}{|W|} \label{step:utilfunction:smoothness:2} \\
                                 &= \alpha \leq \alpha \cdot |S| \label{step:utilfunction:smoothness:3}
\end{align}
The inequality in step~\ref{step:utilfunction:smoothness:2} holds by using the result of step~\ref{step:utilfunction:smoothness:1} and inequality in step~\ref{step:utilfunction:smoothness:3} holds trivially as $|S| \geq 1$. Hence, the smoothness parameter of the function $\lambda_{f}$ is bounded by $1$.
\end{proof}

\begin{lemma} \label{lemma:utilfunction:diversification}
Utility function $f$ in Equation~\ref{eq:utilfunction} satisfies the properties of diversification, i.e. has bounded $\Upsilon_{f}$.
\end{lemma} 
\begin{proof}[\bf{Proof}]
For any given set of users $S$ and any new user $v \in W \setminus S$, let us define $\alpha = \operatorname*{min}_{s \in S} D(s, v)$. The goal is to show that $f(S \cup {v}) - f(S) \leq \Upsilon_{f} \cdot \alpha$ always holds for a fixed and bounded $\Upsilon_{f}$.
\\

\noindent
Again, let us consider the function $f_w$ and consider the marginal of adding $v$ to $S$, given by $f_w(S \cup {v}) - f_w(S)$. When $v$ is not the nearest user to $w$ in the set $S \cup \{v\}$, we have $f(S \cup \{v\}) - f(S) = 0$. Let's consider the case where $f_w(S \cup {v}) - f(S) > 0$, \emph{i.e.}, $v$ is assigned as the nearest user to $w$ from the set $S \cup \{v\}$, given by $v = \operatorname*{min}_{s \in S \cup \{v\} \cup X} D(o_s, o_w)\Big)$. Let us denote the nearest user assigned to $w$ before adding $v$ to the set by $v'$. Then, we have:
\begin{align}
f_w(S \cup {v}) - f_w(S) &= \frac{1}{|W|}\big(D(s, v') - D(s, v)\big) \notag \\ 
                         &\leq \frac{D(v, v')}{|W|} \label{step:utilfunction:diversification:1} \\
                         &\leq \frac{\alpha}{|W|} \label{step:utilfunction:diversification:2}                         
\end{align}
Step~\ref{step:utilfunction:diversification:1} uses the triangular inequality of the underlying metric space. In step~\ref{step:utilfunction:diversification:2}, the inequality holds by the definition of $\alpha$.  Then, we have
\begin{align}
f(S \cup {v}) - f(S) &= \sum_{w \in W} \big(f_w(S \cup {v}) - \sum_{w \in W} f(S)\big) \notag \\
                     &\leq \sum_{w \in W} \frac{\alpha}{|W|} \label{step:utilfunction:diversification:3} \\
                     &= \alpha \notag
\end{align}
The inequality in step~\ref{step:utilfunction:diversification:3} holds by using the result of step~\ref{step:utilfunction:diversification:2}. Hence, the diversification parameter of the function $\Upsilon_{f}$ is bounded by $1$.
\end{proof}

\begin{proof}[\bf{Proof of Lemma~\ref{lemma:utilfunction}}]
The proof directly follows from the results in Lemmas~\ref{lemma:utilfunction:submodular}~\ref{lemma:utilfunction:smoothness}~\ref{lemma:utilfunction:diversification}.
\end{proof}

\section{Proof of Theorem~\ref{theorem:randgreedy}}
\begin{proof}[\bf{Proof of Theorem~\ref{theorem:randgreedy}}]
Let $S^\opt$ be the set returned by $\opt$ for Problem~\ref{eq:opt} without the privacy constraints. By the hypothesis of the theorem, for each of the element $s \in S^\opt$, the $\alpha_{rg}$ neighborhood of $s$ contains a set of at least $\sfrac{1}{r} \cdot \log(\sfrac{B}{\epsilon})$ users. Furthermore, by hypothesis, these sets of size at least $\sfrac{1}{r} \cdot \log(\sfrac{B}{\epsilon})$ can be constructed to be mutually disjoint for every $s, s' \in S^\opt$, let us denote these mutually disjoint sets by $\widetilde{N}_{\alpha_{rg}}(s)$. Formally, this means that for $s \in S^\opt$, we have $|\widetilde{N}_{\alpha_{rg}}(s)| \geq \sfrac{1}{r} \cdot \log(\sfrac{B}{\epsilon})$ and for any pairs of $s, s' \in S^\opt$, we have $\widetilde{N}_{\alpha_{rg}}(s) \cap \tilde{N}_{\alpha_{rg}}(s) = \emptyset$.
\\

\noindent
Recall that the simpler version of \randgreedy first samples the users from $W$ at rate $r$ to create a subset $\widetilde{W}$ such that $|\widetilde{W}| = |W| \cdot r$. We first show that sampling at a rate $r$ by \randgreedy ensures that with high probability (given by $1-\epsilon$), at least one user is sampled from $\widetilde{N}_{\alpha_{rg}}(s)$ for each of the $s \in S^\opt$. Consider the process of sampling for $s$ and $\widetilde{N}_{\alpha_{rg}}(s)$. Each of the users in $\widetilde{N}_{\alpha_{rg}}(s)$ has probability of being sampled given by $r$. Hence, the probability that none of the users in $\widetilde{N}_{\alpha_{rg}}(s)$ are included in $\widetilde{W}$ for a given $s$ is given by:
\begin{align}
P\big(\widetilde{N}_{\alpha_{rg}}(s) \cap \widetilde{W} = \emptyset\big) &= (1-r)^{\sfrac{1}{r} \cdot \log(\sfrac{B}{\epsilon})} \notag \\
                                                                           &\leq e^{-\log(\sfrac{B}{\epsilon})} \notag \\
                                                                           &=  \sfrac{\epsilon}{B} \notag
\end{align}
By using union bound, the probability that none of the users in $\widetilde{N}_{\alpha_{rg}}(s)$ gets included in $\widetilde{W}$ for any $s \in S^\opt$ is bounded by $\epsilon$ (given by $B \cdot \sfrac{\epsilon}{B}$). Hence, with probability at least $1 - \epsilon$, the sampled set $\widetilde{W}$ contains at least one user from $\widetilde{N}_{\alpha_{rg}}(s)$ for every $s \in S^\opt$.
\\

\noindent
This is equivalent to saying that, with probability at least $1 - \epsilon$, the $\widetilde{W}$ contains a set $\widetilde{S}^\opt_{\alpha_{rg}}$ that can be obtained by replacing every $s \in S^\opt$ with some $w \in N_{\alpha_{rg}}(s)$, and hence $f(\widetilde{S}^\opt_{\alpha_{rg}}) \geq f(S^\opt) - \alpha_{rg} \cdot \lambda_{f} \cdot B$ (by using the definition of smoothness property). And, running the \greedy on $\widetilde{W}$ ensures that the utility obtained is at least $(1 - \sfrac{1}{e}) \cdot f(\widetilde{S}^\opt_{\alpha_{rg}})$. Hence, with probability at least $(1 - \epsilon)$,
\begin{align}
\mathbb{E}[f(\randgreedy)] &\geq (1 - \sfrac{1}{e}) \cdot f(\widetilde{S}^\opt_{\alpha_{rg}}) \notag \\
                           &\geq (1 - \sfrac{1}{e}) \cdot \big(f(\opt) - \alpha_{rg} \cdot \lambda_{f} \cdot B\big) \notag
\end{align}
\end{proof}

\vspace{-4mm}
\section{Proof of Theorem~\ref{theorem:spgreedy}}
\begin{proof}[\bf{Proof of Theorem~\ref{theorem:spgreedy}}]
Let $S^\greedy$ be the set returned by $\greedy$ for Problem~\ref{eq:opt} without the privacy constraints. By the hypothesis of the theorem, for each of the element $s \in S^\greedy$, the $\alpha_{spg}$ neighborhood of $s$ contains a set of at least $\sfrac{1}{r}$ users. The loss of utility for the procedure \spgreedy compared w.r.t to \greedy at iteration $i$ can be attributed to two following reasons: $(1)$ obfuscation of $s^*_i$ with set $\psi(s^*_i)$ to select $\widetilde{s^*_i}$, where the size of $\psi(s^*_i)$ is $\sfrac{1}{r}$, and $(2)$ removal of the entire set $\psi(s^*_i)$ for further consideration. We analyze these two factors separately to get the desired bounds on the utility of \spgreedy.
\\

\noindent
We being by stating a more general result on the approximation guarantees of \greedy from \cite{krause05note} when the submodular objective function can only be evaluated approximately within an absolute error of $\epsilon$. Results from \cite{krause05note} states that the utility obtained by this noisy greedy selection is guaranteed to be at least $\big((1 - \sfrac{1}{e}) \cdot \opt - 2 \cdot \epsilon \cdot B\big)$, where $B$ is the budget. 
\\

\noindent
Now, consider an alternate procedure that operates similar to \spgreedy, by obfuscating $s^*_i$ with set $\psi(s^*_i)$ to pick $\widetilde{s^*_i}$ at each iteration $i$. However, this alternate procedure does not eliminate the entire set of users $\psi(s^*_i)$ from the pool, but only removes $\widetilde{s^*_i}$. Instead, it tags the users of $\psi(s^*_i) \setminus \{\widetilde{s^*_i}\}$ as $<$$invalid,i$$>$, \emph{i.e.} these users are marked as \emph{invalid} and are tagged with the iteration $i$ at which they became invalid (in case a user was already marked as invalid, the iteration tag is not updated). Let us denote this alternate procedure by $\overline{\spgreedy}$. This can alternatively be viewed as similar to \greedy, though it can pick the user at every iteration only approximately, because of the noise added by obfuscation. We now bound the absolute value of this approximation error at every iteration.  As $s^*_i$ is obfuscated with a set of users of size $\sfrac{1}{r}$ nearest to  $s^*_i$ from the hypothesis of the theorem, we are certain that set $\psi(s^*_i)$ is contained within a radius of $\alpha_{spg}$ neighborhood. Now, from the smoothness assumptions, the maximum absolute error that could be introduced by the obfuscation compared to greedy selection (\emph{i.e.} the difference in marginal utilities of $s^*_i$ and $\widetilde{s^*_i}$) at a given iteration $i$ is bounded by $\lambda_{f} \cdot \alpha_{spg}$. Hence, the utility obtained by $\overline{\spgreedy}$ can be lower-bounded as:
\begin{align}
f(\overline{\spgreedy}) \geq (1 - \sfrac{1}{e}) \cdot \opt - 2 \cdot \lambda_{f} \cdot \alpha_{spg} \cdot B \label{step:spgreedy:1}
\end{align}
Next, we consider the loss associated with the removal of entire set $\psi(s^*_i)$ at iteration $i$. Let us consider the execution of $\overline{\spgreedy}$ and let $l+1$ be the first iteration when the obfuscation set $\psi(s^*_{l+1})$ created by the procedure contains at least one element marked as 
\emph{invalid}, with the associated iteration of invalidity as $k$. Note that when $l+1 > B$, there is no loss associated with this step of removing $\psi(s^*_i)$ and hence we only consider the case when $l+1 \leq B$. As the users are embedded in euclidean space, this means that the $\alpha_{spg}$ centered around $s^*_{l+1}$ and $s^*_{k}$ overlaps and hence $D(s^*_{l+1}, s^*_{k}) \leq 2 \cdot \alpha_{spg}$. From the diversification assumption, this means that the marginal utility of $s^*_{l+1}$ cannot be more than $2 \cdot   \upsilon_{f} \cdot \alpha_{spg}$. And, furthermore, the submodularity ensures that for all $j > l+1$, the marginal utility of users selected can only be lesser than the marginal utility of $s^*_{l+1}$.
\\

\noindent
Let us consider a truncated version of $\overline{\spgreedy}$ that stops after $l$ steps, denoted by $\overline{\spgreedy}_{V}$, where $V$ denotes the fact that this procedure is always valid as it never touches \emph{invalid} marked users. The utility of the truncated version can be lower-bounded as follows:
\begin{align}
f(\overline{\spgreedy}_{V}) &\geq f(\overline{\spgreedy}) - (B - l) \cdot  (2 \cdot   \upsilon_{f} \cdot \alpha_{spg}) \notag \\
                        &\geq f(\overline{\spgreedy}) -  2 \cdot \upsilon_{f} \cdot \alpha_{spg} \cdot B \notag \\
                        &\geq (1 - \sfrac{1}{e}) \cdot \opt -  (2 \cdot \lambda_{f} + 2 \cdot \upsilon_{f}) \cdot \alpha_{spg} \cdot B \label{step:spgreedy:2}
\end{align}
The step~\ref{step:spgreedy:2} follows by using the result in step~\ref{step:spgreedy:1}. For the first $l$ iterations, the execution of the mechanism \spgreedy is exactly same as $\overline{\spgreedy}_{V}$. Hence, \spgreedy  acquires utility at least that acquired by $\overline{\spgreedy}_{V}$, which completes the proof.
\end{proof}


%% file: stochasticPrivacy.bbl
\begin{thebibliography}{}

\bibitem[\protect\citeauthoryear{Adar}{2007}]{2007-_anonymyzing-query-logs}
Adar, E.
\newblock 2007.
\newblock {User 4xxxxx9: Anonymizing query logs}.
\newblock In {\em Workshop on Query Log Analysis at WWW'07}.

\bibitem[\protect\citeauthoryear{Arrington}{2006}]{2006-_aol}
Arrington, M.
\newblock 2006.
\newblock Aol proudly releases massive amounts of private data.
\newblock
  \href{http://techcrunch.com/2006/08/06/aol-proudly-releases-massive-amounts-of-user-search-data/}
  {\nolinkurl{http://techcrunch}} \\
  \href{http://techcrunch.com/2006/08/06/aol-proudly-releases-massive-amounts-of-user-search-data/}
  {\nolinkurl{.com/2006/08/06/aol-proudly-releases}} \\
  \href{http://techcrunch.com/2006/08/06/aol-proudly-releases-massive-amounts-of-user-search-data/}
  {\nolinkurl{-massive-amounts-of-user-search-data/}}.

\bibitem[\protect\citeauthoryear{Bennett \bgroup et al\mbox.\egroup
  }{2011}]{2011-sigir_bennett_location-personalization}
Bennett, P.~N.; Radlinski, F.; White, R.~W.; and Yilmaz, E.
\newblock 2011.
\newblock Inferring and using location metadata to personalize web search.
\newblock In {\em Proc. of SIGIR},  135--144.

\bibitem[\protect\citeauthoryear{Bennett, Svore, and
  Dumais}{2010}]{2010-www_bennett_classification}
Bennett, P.~N.; Svore, K.; and Dumais, S.~T.
\newblock 2010.
\newblock Classification-enhanced ranking.
\newblock In {\em Proc. of WWW},  111--120.

\bibitem[\protect\citeauthoryear{Cooper}{2008}]{2008-_policy-perspective}
Cooper, A.
\newblock 2008.
\newblock A survey of query log privacy-enhancing techniques from a policy
  perspective.
\newblock {\em ACM Trans. Web} 2(4):19:1--19:27.

\bibitem[\protect\citeauthoryear{Feige}{1998}]{1998-_feige_threshold-of-ln-n}
Feige, U.
\newblock 1998.
\newblock A threshold of ln n for approximating set cover.
\newblock {\em Journal of the ACM} 45:314--318.

\bibitem[\protect\citeauthoryear{FTC}{2011}]{2011-ftc_facebook}
FTC.
\newblock 2011.
\newblock {F}{T}{C} charges against {F}acebook.
\newblock \href{http://www.ftc.gov/opa/2011/11/privacysettlement.shtm}
  {\nolinkurl{http://www.ftc.gov/opa/2011/11/privacysettle}} \\
  \href{http://www.ftc.gov/opa/2011/11/privacysettlement.shtm}
  {\nolinkurl{ment.shtm}}.

\bibitem[\protect\citeauthoryear{FTC}{2012}]{2012-ftc_google}
FTC.
\newblock 2012.
\newblock {F}{T}{C} charges against {G}oogle.
\newblock \url{http://www.ftc.gov/opa/2012/08/google.shtm}.

\bibitem[\protect\citeauthoryear{Hassan and
  White}{2013}]{2013-cikm_ryen_models-of-user-satisfaction}
Hassan, A., and White, R.~W.
\newblock 2013.
\newblock Personalized models of search satisfaction.
\newblock In {\em Proc. of CIKM},  2009--2018.

\bibitem[\protect\citeauthoryear{Kaufman and
  Rousseeuw}{2009}]{2009-book_kaufman_finding}
Kaufman, L., and Rousseeuw, P.~J.
\newblock 2009.
\newblock {\em Finding groups in data: an introduction to cluster analysis},
  volume 344.
\newblock Wiley. com.

\bibitem[\protect\citeauthoryear{Krause and Guestrin}{2005}]{krause05note}
Krause, A., and Guestrin, C.
\newblock 2005.
\newblock A note on the budgeted maximization on submodular functions.
\newblock Technical Report CMU-CALD-05-103, Carnegie Mellon University.

\bibitem[\protect\citeauthoryear{Krause and
  Guestrin}{2007}]{2007-aaai_krause_observation-selection}
Krause, A., and Guestrin, C.
\newblock 2007.
\newblock Near-optimal observation selection using submodular functions.
\newblock In {\em Proc. of AAAI, Nectar track}.

\bibitem[\protect\citeauthoryear{Krause and
  Horvitz}{2008}]{2008-aaai_krause_privacy-personalization}
Krause, A., and Horvitz, E.
\newblock 2008.
\newblock A utility-theoretic approach to privacy and personalization.
\newblock In {\em Proc. of AAAI}.

\bibitem[\protect\citeauthoryear{Krause and
  Horvitz}{2010}]{2010-jair_krause_privacy-personalization}
Krause, A., and Horvitz, E.
\newblock 2010.
\newblock A utility-theoretic approach to privacy in online services.
\newblock {\em Journal of Artificial Intelligence Research (JAIR)} 39:633--662.

\bibitem[\protect\citeauthoryear{Mirzasoleiman \bgroup et al\mbox.\egroup
  }{2013}]{2013-nips_krause_distributed}
Mirzasoleiman, B.; Karbasi, A.; Sarkar, R.; and Krause, A.
\newblock 2013.
\newblock Distributed submodular maximization: Identifying representative
  elements in massive data.
\newblock In {\em Proc. of NIPS}.

\bibitem[\protect\citeauthoryear{Narayanan and
  Shmatikov}{2008}]{2008-_netflix-deanonymization}
Narayanan, A., and Shmatikov, V.
\newblock 2008.
\newblock Robust de-anonymization of large sparse datasets.
\newblock In {\em Proc. of the IEEE Symposium on Security and Privacy},
  111--125.

\bibitem[\protect\citeauthoryear{Nemhauser, Wolsey, and
  Fisher}{1978}]{1978-_nemhauser_submodular-max}
Nemhauser, G.; Wolsey, L.; and Fisher, M.
\newblock 1978.
\newblock An analysis of the approximations for maximizing submodular set
  functions.
\newblock {\em Math. Prog.} 14:265--294.

\bibitem[\protect\citeauthoryear{Olson, Grudin, and
  Horvitz}{2005}]{2005-chi_privacy-preferences}
Olson, J.; Grudin, J.; and Horvitz, E.
\newblock 2005.
\newblock A study of preferences for sharing and privacy.
\newblock In {\em Proc. of CHI}.

\bibitem[\protect\citeauthoryear{Singla and
  White}{2010}]{2010-www_singla_click-features}
Singla, A., and White, R.~W.
\newblock 2010.
\newblock Sampling high-quality clicks from noisy click data.
\newblock In {\em Proc. of WWW},  1187--1188.

\bibitem[\protect\citeauthoryear{Technet}{2012}]{2012-technet}
Technet.
\newblock 2012.
\newblock Privacy and technology in balance.
\newblock
  \href{http://blogs.technet.com/b/microsoft_on_the_issues/archive/2012/10/26/privacy-and-technology-in-balance.aspx}
  {\nolinkurl{http://blogs.technet.com/b/microsoft}} \\
  \href{http://blogs.technet.com/b/microsoft_on_the_issues/archive/2012/10/26/privacy-and-technology-in-balance.aspx}
  {\nolinkurl{_on_the_issues/archive/2012/10/26/privacy}} \\
  \href{http://blogs.technet.com/b/microsoft_on_the_issues/archive/2012/10/26/privacy-and-technology-in-balance.aspx}
  {\nolinkurl{-and-technology-in-balance.aspx}}.

\bibitem[\protect\citeauthoryear{White, Dumais, and
  Teevan}{2009}]{2009-wsdm_white_experts}
White, R.~W.; Dumais, S.~T.; and Teevan, J.
\newblock 2009.
\newblock Characterizing the influence of domain expertise on web search
  behavior.
\newblock In {\em Proc. of WSDM},  132--141.

\bibitem[\protect\citeauthoryear{Wikipedia-comScore}{2006}]{Wikipedia-comScore}
Wikipedia-comScore.
\newblock 2006.
\newblock Com{S}core-\#{D}ata\_collection\_and\_reporting.
\newblock
  \href{http://en.wikipedia.org/wiki/ComScore\#Data_collection_and_reporting}
  {\nolinkurl{http://en.}} \\
  \href{http://en.wikipedia.org/wiki/ComScore\#Data_collection_and_reporting}
  {\nolinkurl{wikipedia.org/wiki/ComScore\#Data_collect}} \\
  \href{http://en.wikipedia.org/wiki/ComScore\#Data_collection_and_reporting}
  {\nolinkurl{ion\_and\_reporting}}.

\bibitem[\protect\citeauthoryear{Xu \bgroup et al\mbox.\egroup
  }{2007}]{2007-_privacy-enhancing}
Xu, Y.; Wang, K.; Zhang, B.; and Chen, Z.
\newblock 2007.
\newblock Privacy-enhancing personalized web search.
\newblock In {\em Proc. of WWW},  591--600.
\newblock ACM.

\end{thebibliography}
